\RequirePackage[l2tabu, orthodox]{nag}
\documentclass[12pt]{article}

\usepackage[nonatbib,final]{nips_2016}

\usepackage{setspace}
\usepackage{makecell}
\newcolumntype{M}[1]{>{\centering\arraybackslash}m{#1}}
\usepackage[T1]{fontenc}

\usepackage{tgtermes}
\usepackage{amsmath}
\usepackage{scalefnt,letltxmacro}
\LetLtxMacro{\oldtextsc}{\textsc}
\renewcommand{\textsc}[1]{\oldtextsc{\scalefont{1.10}#1}}
\usepackage[scaled=0.92]{PTSans}
\usepackage{inconsolata}
\usepackage{mathbbol}
\usepackage{amsthm}
\usepackage{amssymb}

\usepackage{setspace}

\usepackage[usenames,dvipsnames]{xcolor}
\definecolor{shadecolor}{gray}{0.9}

\usepackage[english]{babel}
\usepackage{afterpage}
\usepackage{framed}
\usepackage{nicefrac}

{\endMakeFramed}

\DeclareRobustCommand{\parhead}[1]{\textbf{#1}~}

\usepackage{lineno}

\usepackage{ragged2e}

\usepackage{marginnote}

\newcounter{parcount}

\usepackage{graphicx}
\usepackage[labelfont=bf]{caption}
\usepackage[format=hang]{subcaption}
\usepackage{wrapfig}

\usepackage{booktabs}
\usepackage{multirow}
\usepackage{longtable}

\usepackage{natbib}
\usepackage{bibunits}

\usepackage[algoruled]{algorithm2e}
\usepackage{listings}
\usepackage{fancyvrb}
\fvset{fontsize=\normalsize}

\SetKwComment{Comment}{$\triangleright$\ }{}

\usepackage[colorlinks,linktoc=all]{hyperref}
\usepackage[all]{hypcap}
\hypersetup{citecolor=Violet}
\hypersetup{linkcolor=black}
\hypersetup{urlcolor=MidnightBlue}

\usepackage[nameinlink]{cleveref}

\usepackage[acronym,smallcaps,nowarn]{glossaries}

\usepackage{listings}
\usepackage{lstbayes}
\lstdefinestyle{mystyle}{
    commentstyle=\color{OliveGreen},
    numberstyle=\tiny\color{black!60},
    stringstyle=\color{BrickRed},
    basicstyle=\ttfamily\scriptsize,
    breakatwhitespace=false,
    breaklines=true,
    captionpos=b,
    keepspaces=true,
    numbers=none,
    numbersep=5pt,
    showspaces=false,
    showstringspaces=false,
    showtabs=false,
    tabsize=2
}
\usepackage{enumitem}

\usepackage{centernot}

\DeclareMathOperator*{\argmin}{arg\,min}

\crefname{lemma}{lemma}{lemmas}
\Crefname{lemma}{Lemma}{Lemmas}
\crefname{thm}{theorem}{theorems}
\Crefname{thm}{Theorem}{Theorems}
\crefname{prop}{proposition}{propositions}
\Crefname{prop}{Proposition}{Propositions}
\crefname{defn}{definition}{definitions}
\Crefname{defn}{Definition}{Definitions}
\creflabelformat{equation}{#1#2#3}
\crefname{equation}{eq.}{eqs.}
\Crefname{equation}{Eq.}{Eqs.}
\Crefname{section}{\S}{\S}
\crefname{figure}{fig.}{figs.}
\Crefname{figure}{Fig.}{Figs.}
\crefname{algorithm}{alg.}{algs.}
\Crefname{algorithm}{Alg.}{Algs.}

\newtheorem{thm}{Theorem} 
\newtheorem{defn}{Definition} 
\newtheorem{prop}[thm]{Proposition}
\newtheorem{lemma}[thm]{Lemma}

\newcommand\dif{\mathop{}\!\mathrm{d}}

\newcommand{\g}{\, | \,}
\newcommand{\s}{\, ; \,}

\newcommand{\E}[2]{\mathbb{E}_{#1}\left[#2\right]}

\usepackage{booktabs,arydshln}
\makeatletter
\def\adl@drawiv#1#2#3{%
        \hskip.5\tabcolsep
        \xleaders#3{#2.5\@tempdimb #1{1}#2.5\@tempdimb}%
                #2\z@ plus1fil minus1fil\relax
        \hskip.5\tabcolsep}
\newcommand{\cdashlinelr}[1]{%
  \noalign{\vskip\aboverulesep
           \global\let\@dashdrawstore\adl@draw
           \global\let\adl@draw\adl@drawiv}
  \cdashline{#1}
  \noalign{\global\let\adl@draw\@dashdrawstore
           \vskip\belowrulesep}}
\makeatother

\newacronym{KL}{kl}{Kullback-Leibler}
\newacronym{EO}{eo}{equal opportunities}
\newacronym{AA}{aa}{affirmative action}
\newacronym{FTU}{ftu}{fairness through unawareness}
\newacronym{PS}{ps}{personal statement}
\newacronym{SEM}{sem}{structural equation model}
\newacronym{ML}{ml}{machine learning}

\newacronym{CS}{cs}{computer science}

\usepackage{tikz}
\usetikzlibrary{bayesnet}
\usepackage{pgfplots}
\pgfplotsset{compat=newest}
\pgfplotsset{plot coordinates/math parser=false}
\usepgfplotslibrary{statistics}

\pgfdeclarelayer{edgelayer}
\pgfdeclarelayer{nodelayer}
\pgfsetlayers{edgelayer,nodelayer,main}

\definecolor{hexcolor0xbfbfbf}{rgb}{0.749,0.749,0.749}

\tikzset{>=latex}
\tikzstyle{none}   = [inner sep=0pt]
\tikzstyle{line}   = [ thick, -, shorten <=1pt, shorten >=1pt ]
\tikzstyle{arrow}  = [ thick,  ->, shorten <=1pt, shorten >=1pt ]
\tikzstyle{ardash} = [ thick dotted, ->, shorten <=1pt, shorten >=1pt ]

\tikzstyle{empty}=[circle,opacity=0.0,text opacity=1.0,minimum width=4pt,minimum height=4pt]
\tikzstyle{box}=[rectangle,fill=White,draw=Black]
\tikzstyle{filled}=[circle,fill=hexcolor0xbfbfbf,draw=Black]
\tikzstyle{hollow}=[circle,fill=White,draw=Black]
\tikzstyle{param}=[rectangle,fill=Black,draw=Black,inner sep=0pt,minimum width=4pt,minimum height=4pt]
\tikzstyle{paramhollow}=[rectangle,fill=White,draw=Black,inner sep=0pt,minimum
width=4pt,minimum height=4pt]

\newcommand*\circled[1]{\tikz[baseline=(char.base)]{
            \node[shape=circle,draw,inner sep=2pt] (char) {#1};}}

\usepackage{times}
\usepackage{adjustbox}
\usepackage{tcolorbox}
\usepackage{enumitem}
\usepackage{natbib}

\usepackage[defaultlines=4,all]{nowidow}

\newcommand{\ML}{\mathrm{ml}}
\newcommand{\acc}{\mathrm{ml}}
\newcommand{\eo}{\mathrm{eo}}
\newcommand{\rmaa}{\mathrm{aa}}
\newcommand{\new}{\mathrm{new}}

\everypar=\expandafter{\the\everypar\loosness=-1 }
\title{Equal Opportunity and Affirmative Action \\via
Counterfactual Predictions}

\author{
  Yixin Wang\\
  Columbia University\\
  \And
  Dhanya Sridhar\\
  Columbia University\\
  \And
  David M.~Blei\\
  Columbia University\\
}

\begin{document}

\maketitle

\begin{bibunit}[alp]

\begin{abstract}

  \Gls{ML} can automate decision-making by learning to predict
  decisions from historical data.  However, these predictors may
  inherit discriminatory policies from past decisions and reproduce
  unfair decisions.  In this paper, we propose two algorithms that
  adjust fitted \gls{ML} predictors to make them fair.  We focus on
  two legal notions of fairness: (a) providing \gls{EO} to individuals
  regardless of sensitive attributes and (b) repairing historical
  disadvantages through \gls{AA}.  More technically, we produce fair
  \gls{EO} and \gls{AA} predictors by positing a causal model and
  considering counterfactual decisions.  We prove that the resulting
  predictors are theoretically optimal in predictive performance while
  satisfying fairness.  We evaluate the algorithms, and the trade-offs
  between accuracy and fairness, on datasets about admissions, income,
  credit,~and~recidivism.

\end{abstract}


\section{Introduction}
\label{sec:introduction}

\glsresetall

\Gls{ML} methods can automate costly decisions by learning from
historical data. For example, \gls{ML} algorithms can assess the risk of
recidivism to decide bail releases or predict admissions to determine
which college applicants should be further reviewed~\citep{waters2014grade}.

Consider an admissions committee that decides which applicants to
accept to a university.  Given historical data about the applicants
and the admission decisions, an \gls{ML} algorithm can learn to predict who
will be admitted and who will not.  The \gls{ML} predictor might accurately
simulate the committee's decisions, but it might also inherit some of
the undesirable properties of the committee.  If the committee was
systematically and unfairly biased then its \gls{ML} replacement will be
biased as well.

In this paper, we develop algorithms for learning \gls{ML} predictors that
are both accurate and fair.  Our methods maintain as much fidelity as
possible to the decisions of the committee, but they adjust the
predictor to ensure that the algorithmic decisions are fair.

We focus on two legal notions of fairness: \gls{EO} and \gls{AA}.  An
\gls{EO} decision provides the same opportunities to similar
candidates with different backgrounds~\citep{barocas2016big}. In
admissions, for example, an \gls{EO} decision admits applicants based
only on their qualifications; applicants with the same merit, such as
a test score, should have the same chance of admission regardless of
their race or sex.

\glsreset{AA}

Some applicants might have lower test scores only because they belong
to a historically disadvantaged demographic group, one that has less
access to opportunities.  The legal notion of \gls{AA} acknowledges
these historical disadvantages \citep{anderson2002integration}. For
example, \gls{AA} will correct for situations where male applicants
have easier access to extra test preparation.  (Notice that, by
definition, an \gls{AA} decision is not an \gls{EO} decision because
\gls{AA} decisions provide advantages based on group membership.)

Using these two notions of fairness, we develop algorithms that adjust
\gls{ML} predictors to be fair predictors.  The first produces algorithmic
decisions that ensure equal opportunities for disadvantaged groups;
the second produces decisions that exercise affirmative action to
repair existing disparities.  We prove the algorithms are
theoretically optimal---they provide as good predictions as possible
while still being fair.

In detail, we use ideas from causal machine
learning~\citep{pearl2009causality,peters2017elements}.  First we
posit a causal model of the historical decision-making process.  In
that model, the decision can be affected by all the attributes, both
sensitive and non-sensitive, and the sensitive attributes can also
affect the non-sensitive ones (e.g., sex can affect test scores).
Then we consider algorithmic decisions as causal variables, ones whose
values are functions of the attributes.  Once framed this way, we can
ask counterfactual questions about an algorithm, e.g., ``what would
the algorithm decide for this applicant if she had achieved the same
test score but was a male?''  We operationalize \gls{EO} and \gls{AA}
fairness as probabilistic properties of such counterfactual decisions.
(\gls{AA}-fairness is the same as counterfactual fairness, proposed in
\citet{kusner2017counterfactual}.)  Finally, we derive a method that
adjusts a classical \gls{ML} decision-maker to be \gls{EO}-fair, and
that adjusts an \gls{EO} decision-maker to exercise \gls{AA}.

We study our algorithms on simulated admissions data and on three
public datasets, about income, credit, and recidivism.  We examine the
gap in accuracy between classical prediction and
\gls{EO}/\gls{AA}-fair prediction, and we investigate the degree of
unfairness that classical prediction would exhibit.  Compared to other
approaches to \gls{EO}/\gls{AA}-fairness, ours provide higher accuracy
while~remaining~fair.



\parhead{Related Work.}  Several threads of related work draw on
causal models to formalize fairness.  \citet{kilbertus2017avoiding}
relate paths and variables in causal models to violations of \gls{EO}
and \gls{AA}.  \citet{zhang2018fairness} decompose existing fairness
measures into discrimination along different paths
in the causal graph, and propose fair decision algorithms that satisfy
the path-specific criteria.  \citet{kusner2017counterfactual}
introduce counterfactual fairness and an algorithm to satisfy it.  Our
work also uses causal graphs to formalize the \gls{EO} and \gls{AA}
criteria, but our fair algorithms are designed to minimally correct
fitted \gls{ML} predictors.  Like
FairLearning~\cite{kusner2017counterfactual}, our \gls{AA}-fair
algorithm also satisfies counterfactual fairness, but our algorithm
has higher fidelity to the historical decisions.

Other research has proposed a variety of statistical definitions of
fairness for supervised
learning~\citep{hashimoto2018fairness,chen2019fairness,kleinberg2016inherent,chouldechova2017fair,zafar2017fairness}. \citet{hardt2016equality}
and \citet{kallus2018residual} show, however, that these approaches
cannot distinguish between different mechanisms of discrimination and
they are not robust to changes in the data distribution.  By building
on causal machine learning, our work separates direct discrimination
from indirect bias, and it is robust to changes in the data
distribution.

Another line of work focuses on data preprocessing.
\citet{zemel2013learning} learn intermediate representations that
remove information about the sensitive attributes;
\citet{calders2010three} and \citet{kamiran2012data} reweight training
data to control false positive/negative rates. In contrast, our
\gls{EO}-fair and \gls{AA}-fair algorithms adjust \gls{ML} predictions
without changing the training data.

\citet{wang2019repairing} modify black-box \gls{ML} predictors
to correct for historical disadvantages, as in \gls{AA}. 
They define adjusted distributions over attributes for the disadvantaged group
and solve optimal transport problems to repair predictors.
Our predictors also minimally adjust \gls{ML} predictors
but we derive adjustments using causality.

Finally, \citet{dwork2012fairness} define individual fairness, which
formalizes \gls{EO} without causality.  Given a distance metric
between individuals, it requires that similar individuals receive
similar decisions.  Our \gls{EO}-fair algorithm recovers this
requirement without relying on an explicit distance metric, which can
be difficult to construct.  \citet{dwork2012fairness} also specify
group-level \gls{AA} fairness as independence between the sensitive
attribute and decision (demographic parity).  We exercise \gls{AA} for
each individual while also satisfying demographic parity.



\parhead{Contributions.} We develop optimal \gls{EO} and \gls{AA}
algorithms.  These algorithms modify existing \gls{ML} predictors to produce
\gls{EO}-fair and \gls{AA}-fair decisions, and we prove these
decisions maximally recover the \gls{ML} predictors under the fairness
constraints. While existing approaches must omit descendants of the
sensitive attributes to construct \gls{AA}-fair decisions
\citep{kusner2017counterfactual,kilbertus2017avoiding}, our
\gls{AA} algorithm uses all available attributes and is still
\gls{AA}-fair. Empirically, we show that our \gls{EO} and \gls{AA}
algorithms produce better decisions than existing approaches that
satisfy the same fairness criteria.



\section{Assessing fairness with counterfactuals}
\label{sec:criteria}

\begin{figure}
  \begin{minipage}{.5\textwidth}
    \centering
    \captionsetup{width=.9\linewidth}
    \footnotesize
  \begin{center}
    \begin{tabular}{lcccccc}
      \toprule
      \textbf{ID}
      & \textbf{Sex}
      & \textbf{Test}
      & \textbf{Admit}
      & $\hat{y}^{\ML}$
      & $\hat{y}^{\eo}$
      & $\hat{y}^{\rmaa}$ \\
      \midrule
      1 & f &  54 & yes \\
      2 & m & 66 & no \\
      \vdots & \vdots & \vdots &\vdots \\
      5000 & f & 44 & no \\
      \midrule
      A & f & 85 & ? & 0.67 & 0.77 & 0.78 \\
      B & m & 85 & ? & 0.84 & 0.77 & 0.76 \\
      C & f & 65 & ? & 0.57 & 0.69 & 0.70 \\
      \bottomrule
    \end{tabular}
  \end{center}
    \caption{We simulate an unfair admissions
      process that violates both \gls{EO} and
      \gls{AA}. This table shows the training data
      (1-5000) and three new applicants (A, B,
      C). \label{tab:admissions}}
  \end{minipage}
  \begin{minipage}{.5\textwidth}
    \centering
    \captionsetup{width=.9\linewidth}
    \includegraphics[scale=0.4]{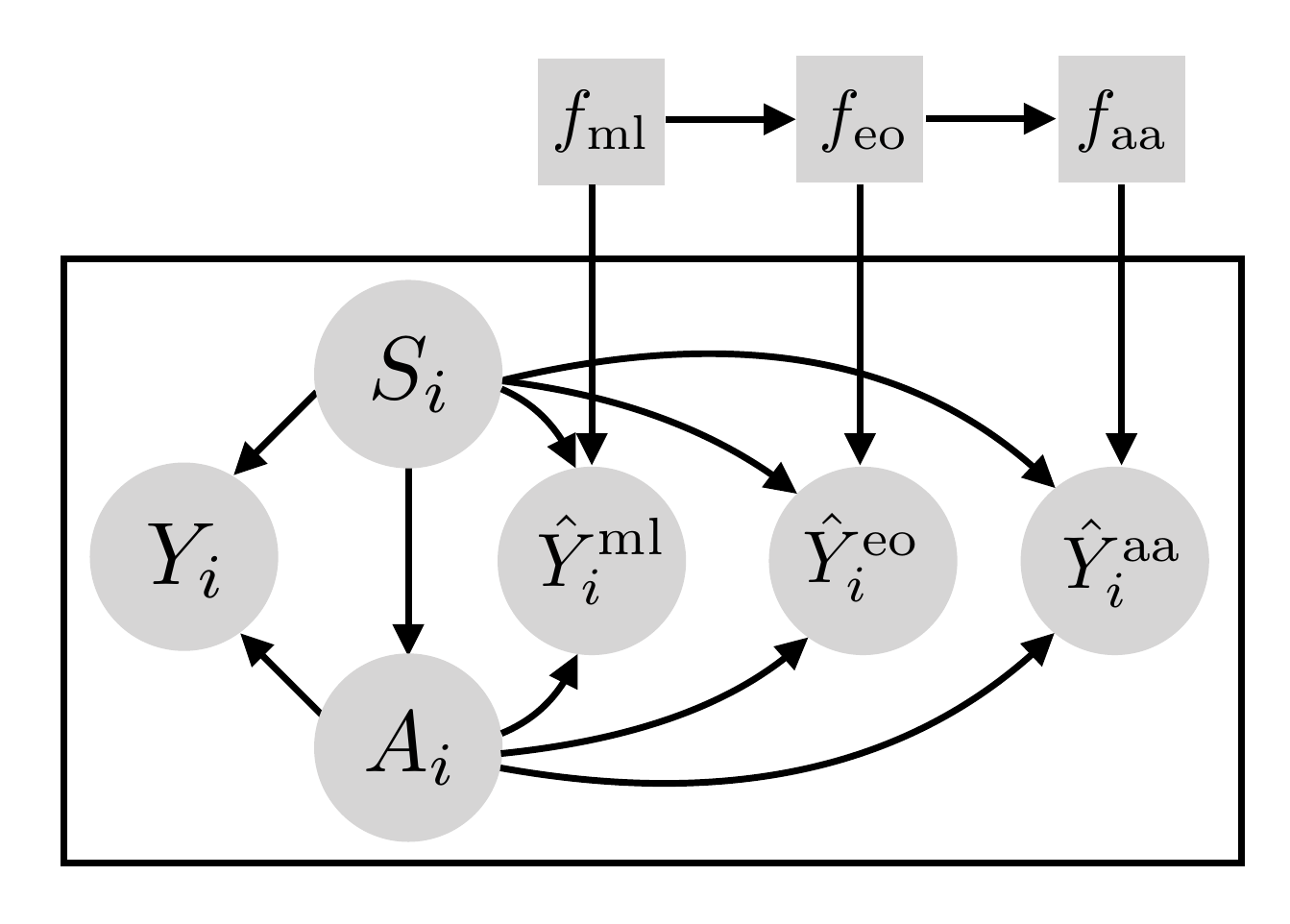}
    \caption{Classical \gls{ML}, \gls{EO} and \gls{AA} decisions
      are causal variables. \gls{EO} predictors upgrade
      classical \gls{ML} ones, and \gls{AA} adjust \gls{EO} ones. \label{fig:eoaagraph}}
  \end{minipage}
\end{figure}

As a running example, consider automating the admissions process at a
university.  Using a dataset of past admissions, the goal is to
algorithmically compute the admissions decision for new applicants.
The dataset contains $n$ applicants and the committee's decision about
whether to admit them.  Each applicant has demographic attributes,
such as sex, religion, or ethnicity, that the university has deemed
sensitive.  Each applicant also has other attributes, such as a
standardized test score or grade point average, that are not
sensitive.

For simplicity, suppose each applicant has two attributes, their sex
and their score on a standardized test.  The sex of the applicant is a
sensitive attribute; the test score is not. (Our algorithms extend to
multiple attributes but this simple case illustrates the concepts.)
\Cref{tab:admissions} shows an example of~such~data.~\looseness=-1

We will use a causal model of this decision-making process to define
two notions of fairness.

\parhead{A causal model of the decision.} First posit a \emph{causal}
model of the admissions process, such as the causal graph in
\Cref{fig:eoaagraph}.  It assumes that an applicant's sex $S$ and test
score $A$ affect the admissions decision $Y$.  For example, the
admissions committee might unfairly prefer to admit males.  The model
also assumes the applicant's sex can affect the test score.  For
example, female applicants might receive less exposure to test
preparation, a disparity that results in lower scores.\footnote{We
  acknowledge these assumptions might be debatable; it is up to the
  user of our method to posit a sensible causal model of the
  historical decision process.}  (This simple model illustrates the
core ideas; \Cref{sec:fair-decisions} discusses extensions to more
complicated causal graphs.)

The data in \Cref{tab:admissions} matches the assumptions.  It comes
from the following structural model,
\begin{align}
  \label{eq:admissions-model}
  \begin{split}
    s_i &\sim \mathrm{Bernoulli}(0.5),\\
    a_i \g s_i &= \max(0,
    \min(\lambda \cdot s_i + 100 \cdot \varepsilon, 100));
    \quad
    \varepsilon \sim \textrm{Uniform}[0,1] \\
    y_i \g s_i, a_i & \sim\mathrm{Bernoulli}(\sigma(-1.0 + \beta_a\cdot
    a + \beta_s \cdot s)).
  \end{split}
\end{align}
(We threshold the test score to mimic real-world test scores that are
usually bounded.)  We set $\beta_a = 2.0$ and $\beta_s = 1.0$ so that
for the same test score the committee is more likely to admit males
than females. Moreover, we set $\lambda = 2.0$ to model that females
perform less well on the test.

With data like \Cref{tab:admissions} and a causal model, we can
estimate the parameters of the functions that govern each variable in
\Cref{fig:eoaagraph}.  (This estimation makes the usual assumptions
about causal estimation, i.e., no open backdoor paths
\citep{pearl2009causality}.)

The causal model implies counterfactuals, how the distribution of a
variable changes when we intervene on others.  For example, we can
consider the counterfactual decision of a female applicant had she
been a male. \footnote{Even when sensitive attributes are immutable,
  we can define hypothetical interventions to articulate the
  counterfactuals \citep{greiner2011causal}. In this example,
  ''intervening on sex'' means to make a person born of a different
  sex.}  Denote $Y_i(s,a)$ to be the counterfactual decision of the
$i$th applicant when we set their test score to be $A_i=a$ and sex to
be $S_i=s$.

Marginal distributions of counterfactuals provide an alternative to
the \textit{do} notation, $P(Y_i(s,a)) = P(Y_i \, ; \,
\textrm{do}(S_i=s, A_i=a))$.  But counterfactual distributions can
also be more complex: for example, $P(Y_i(\textrm{female}) \g Y_i =
\textrm{Yes}, S_i=\textrm{male})$ is the probability of applicant
$i$'s counterfactual admission if he was female, given that he is
(factually) male and was admitted.  See Ch. 7 of
\citet{pearl2009causality} for a treatment of how causal models imply
distributions of counterfactuals.

We use counterfactual distributions to assess the fairness of
decisions, either those made by the committee or those produced by an
algorithm.  We next describe the counterfactual criteria of equal
opportunities and affirmative action.

\parhead{Equal opportunities and its counterfactual criterion.}  A
decision satisfies \gls{EO} if changing the sensitive attribute $S$
does not change the distribution of the decision $Y$.  Consider an
applicant with sex $s$ and test score $a$.  An \gls{EO} decision gives
the same probability of admission even when changing the sex to $s'$
(but keeping the test score at $a$).  In \Cref{tab:admissions},
candidates A and B have the same test score but different sex; they
should receive the same probability of admission.

\begin{defn}[The \gls{EO} criterion]\glsreset{EO} A decision $Y$ satisfies \gls{EO}
  in the sensitive attribute $S$ if
  \begin{align}
    \label{eq:equalopportunities}
    \begin{split}
      Y(s,a)  \g  \{S=s, A=a\}  \stackrel{d}{=} 
      Y(s',a)\g \{S=s, A=a\}
    \end{split}
  \end{align}
  for all $s, s'\in \mathcal{S}$ and $a \in \mathcal{A}$, where
  $\mathcal{S}$ and $\mathcal{A}$ are the domains of $S$ and
  $A$.\footnote{Two random variables are equal in distribution
  $X_1\stackrel{d}{=}X_2$ if they have the same distribution function:
  $P(X_1\leq x)=P(X_2\leq x)$ for all $x$.}
  \label{def:eocriterion}
\end{defn}

The \gls{EO} criterion differs from counterfactual fairness
\citep{kusner2017counterfactual}.  \gls{EO} focuses on the
counterfactual prediction had the applicant had the same non-sensitive
attribute $A$ but a different sensitive attribute $S$; counterfactual
fairness does not hold the attribute $A$ fixed.

The \gls{EO} criterion also differs from equalized odds
\citep{hardt2016equality}. Equalized odds requires the sensitive
attribute $S$ to depend on the decision $Y$ only via its ability to
predict desirable downstream outcomes (e.g., college
success). Equalized odds enforces statistical properties while
\gls{EO} assesses counterfactuals. The \gls{EO} criterion
assumes that the sensitive attribute $S$ does not causally affect the
decision $Y$. (We prove this in \Cref{sec:eothmpf}.)

\glsreset{AA}

\parhead{Affirmative action and its counterfactual criterion.}
\Gls{AA} aims to correct for the historical disadvantages that may be
caused by sensitive attributes.  For example, female applicants may
have had fewer opportunities for extracurricular test preparation,
which leads to lower test scores and a lower likelihood of
admission. Had they been male, they may have had more opportunities
for preparation, achieved a higher score, and been admitted.

The goal of \gls{AA} is to ensure that, all else being equal (e.g.,
effort or aptitude), an applicant would receive the same decision had
they not been in a disadvantaged group.  This is different from
\gls{EO} because it accounts for how the sensitive attribute affects
the other attributes, for example how an applicant's sex affects their
test score.

The \gls{AA} criterion requires that changing the sensitive attribute
$S$, along with the resulting change in the attribute $A$, does not
change the distribution of the decision $Y$.  Consider an applicant
with test score $a$ and sex $s$. Had they belonged to a different
group $s'$, they may have achieved a different test score $A(s')$. The
\gls{AA} criterion requires the counterfactual decision with the
alternate sex $Y(s', A(s'))$ to have the same distribution as the
decision with the current sex $Y(s, A(s))$.  It was first proposed by
\citet{kusner2017counterfactual}, who dubbed it \textit{counterfactual
  fairness}.

\begin{defn}[The \gls{AA} criterion, counterfactual fairness
  \citep{kusner2017counterfactual}] \glsreset{AA} A decision $Y$
  satisfies \gls{AA} if
  \begin{align}
    \label{eq:aacriterion}
    \begin{split}
      Y(s, A(s)) \g \{S=s, A=a\} \stackrel{d}{=} Y(s', A(s'))\g \{S=s,
      A=a\},
    \end{split}
  \end{align}
  for all $s'\in\mathcal{S}$. The random variable $A(s')$ is the
  counterfactual of $A$ under intervention of the sensitive attribute
  $S=s'$.
  \label{def:aacriterion}
\end{defn}
In \Cref{eq:aacriterion} the counterfactual attribute $A(s')$ depends
on the attributes of the applicant, $\{S=s, A=a\}$.  The variable
$A(\textrm{male}) \g \{S=\textrm{female}, A=\textrm{high score}\}$
captures, for example, that female applicants with high test scores
will have even higher test scores had they been male.  Assuming the
causal model in \Cref{fig:eoaagraph}, the \gls{AA} criterion implies
demographic parity \citep{dwork2012fairness,
  kusner2017counterfactual}, which requires the decision be independent
of the sensitive attribute. We note that a decision cannot be both
\gls{EO}-fair and \gls{AA}-fair.

\parhead{Disparate treatment and impact.} We have described two
counterfactual criteria for fairness.  One view on fairness is to seek
to avoid ``disparate treatment'' or ``disparate impact.'' Applicants
face disparate treatment when they are rejected by the biased
admissions committee because of their sex.  Applicants face disparate
impact when they are systematically disadvantaged because they do not
have the same opportunities for test preparation. From this
perspective on fairness, an \gls{EO}-fair decision eliminates
disparate treatment; an \gls{AA}-fair decision addresses disparate
impact.

\section{Constructing fair algorithmic decisions}
\label{sec:fair-decisions}

\Cref{def:eocriterion,def:aacriterion} define fairness in terms
of distributions of counterfactual decisions.  Using the causal model
of the decision-making process, these definitions help evaluate the
fairness of its outcomes.  The criteria can assess any decisions,
including those produced by a real-world process, such as an
admissions committee, or those produced by an algorithm, such as a
fitted \gls{ML} model.

Consider a classical \gls{ML} model that was fit to emulate historical
admissions data; if the historical decisions are not \gls{EO} or
\gls{AA} then neither will be the \gls{ML} decisions.  To this end, we
develop fair \gls{ML} predictors, those that are accurate with respect to
historical data and produce \gls{EO}-fair or \gls{AA}-fair decisions.
The key idea is to treat the algorithmic decision as another causal
variable in the model and then to analyze its fairness properties in
terms of the corresponding counterfactuals.  We will show how to
adjust decisions made by an \gls{ML} model to make them \gls{EO}-fair or
\gls{AA}-fair.

Formally, we use the language of \textit{decisions} and
\textit{predictors}.  Denote an algorithmic decision as $\hat{Y}$, a
causal variable that depends on the attributes $\{s, a\}$.  It comes
from a fitted probabilistic predictor, $\hat{Y} \sim f(s, a)$, where
$f(\cdot)$ is a probability density function. For example, an
admissions decision $\hat{Y} \sim f(s, a)$ is drawn from a Bernoulli
that depends on the attributes (e.g., a logistic regression).  By
considering algorithmic decisions in the causal model, we can infer
algorithmic counterfactuals $\hat{Y}(s,a)$ and assess whether they
satisfy the fairness criteria of \Cref{sec:criteria}.

\parhead{\gls{ML} decisions.}  Consider a classical machine learning
predictor, one that is fit to accurately predict the decision $Y$ from
$S$ and $A$.  In the admissions example, a binary decision, logistic
regression is a common choice.  The predictor $f_{\ML}(s, a)$ draws
the decision from a Bernoulli,
\begin{align*}
  \hat{Y}^{\acc} \g \{S=s, A=a\} \sim \mathrm{Bern}(\sigma(\beta_{A} \cdot
  a + \beta_{S} \cdot s + \beta_0)),
\end{align*}
and the coefficients are fit to maximize the likelihood of the
observed data.  When incorporated into the causal model, $f_{\acc}$
and $\hat{Y}^{\acc}$ are illustrated in \Cref{fig:eoaagraph}.

The \gls{ML} decision $\hat{Y}^{\acc}$ will accurately mimic the historical
data.  However, it will also replicate harmful discriminatory
practices.  The decision $\hat{Y}^{\acc}$ might violate \gls{EO} and
not give equal opportunities to applicants of different sex. Nor will
$\hat{Y}^{\acc}$ exercise \gls{AA}. If policy mandates that admissions
correct for the disparate impact of sex on the test score, then the \gls{ML}
decision cannot be used.

Return to the illustrative simulation in \Cref{tab:admissions}.  We
fit a logistic regression to the training data, which finds
coefficients close to the mechanism that generated the data.
Consequently, when used to form algorithmic decisions, it replicates
the unfair committee.  To see this, consider female applicant A
$(s=\textrm{f}, a=85)$ and male applicant B $(s=\textrm{m}, a=85)$ and
the classical \gls{ML} decisions for each.  For A, her probability of being
admitted is $67\%$.  For B, his probability of admission is $84\%$.
Despite their identical test scores, the female applicant is $17\%$
less likely to get in.

\parhead{\gls{ML} decisions that satisfy equal opportunities.}  We use the
\gls{ML} predictor $f_{\acc}$ to produce an \gls{EO} predictor $f_{\eo}$,
one whose decisions satisfy the \gls{EO} criterion.  Consider an \gls{ML}
predictor $f_{\acc}(s,a)$ and an applicant with attributes
$\{s_{\new},a_{\new}\}$.  Their \gls{EO} decision is
$\hat{Y}^{\eo}(s_{\new}, a_{\new}) \sim f_{\eo}(a_{\new})$, where
\begin{align}
  \label{eq:counterfactualeo}
  f_{\eo}(a_{\new}) &= \int f_{\acc}(s, a_{\new}) p(s) \dif s.
\end{align}
The distribution $p(s)$ is the population distribution of the
sensitive attribute.

The \gls{EO} decision probability holds the non-sensitive attribute
$a_{\new}$ fixed, and takes a weighted average of the \gls{ML} predictor for
the sensitive attribute $s$.  Note all of its ingredients are
computable: the distribution $p(s)$ is estimated from the data and
$f_{\acc}(s, a)$ is the fitted \gls{ML} predictor.

The \gls{EO} decision $\hat{Y}^{\eo} \sim f_{\eo}(a_{\new})$ satisfies
\Cref{def:eocriterion}: applicants with the same test score will have
the same chance of admissions regardless of their sex.  Further, it is
the most accurate among all \gls{EO}-fair decisions.  We prove that it is \gls{EO}-fair and minimally corrects
  $\hat{Y}^{\acc}$..~\looseness=-1

\begin{thm}[\gls{EO}-fairness and optimality of \gls{EO}
  decisions\label{thm:eothm}]
The decision $\hat{Y}^{\eo} \sim f_{\eo}(s, a)$ is \gls{EO}-fair.
Moreover, among all \gls{EO}-fair decisions, $\hat{Y}^{\rm{eo}}$
maximally recovers the ML decision $\hat{Y}^{\acc}$,
\begin{align*}
\hat{Y}^{\rm{eo}} =
\argmin_{Y^{\rm{eo}}
\in\mathcal{Y}^{EO}} \E{S, A}{\mathrm{KL}(P(\hat{Y}^{\acc}(S,
A))||P(Y^{\eo}(S, A)))},
\end{align*}
where $\mathcal{Y}^{EO}$ is the set of \gls{EO}-fair decisions. (The
proof is in \Cref{sec:eothmpf}.)
\end{thm}

The intuition is that the \gls{EO} decision preserves the causal
relationship between the test score $A$ and the \gls{ML} decision
$\hat{Y}^{\acc}$, but it ignores the possible effect of the sensitive
attribute $S$.  In the \textit{do} notation, what this means is that
$P(\hat{Y}^{\eo} ; \textrm{do}(s, a)) = P(\hat{Y}^{\acc} ;
\textrm{do}(a))$ for all $s$. \Cref{eq:counterfactualeo} is the
adjustment formula in \citet{pearl2009causality}, which calculates
$P(\hat{Y}^{\acc} ; \textrm{do}(a))$.

Again return to \Cref{tab:admissions}.  We use the fitted logistic
regression $f_{\ML}(s, a)$ to produce the \gls{EO} predictor.  This
data has equal numbers of women and men, so the weighted average is
$$f_{\eo}(s, a) = 0.5 f_{\ML}(\textrm{male}, a) + 0.5
f_{\ML}(\textrm{female}, a).$$ Using \gls{EO}-fair decisions,
applicants A and B both have a $77\%$ probability of admission.

Why not use \gls{FTU}~\citep{kusner2017counterfactual}?  \gls{FTU}
satisfies \gls{EO} by fitting an \gls{ML} model from the non-sensitive
attribute $A$ to the decision $Y$.  \gls{FTU} decisions are also
$\gls{EO}$-fair; the sensitive attribute is never used.  However,
\Cref{thm:eothm} states that the \gls{EO} decisions of
\Cref{eq:counterfactualeo} are more accurate than \gls{FTU} while
still being $\gls{EO}$-fair.  We show this fact empirically in
\Cref{sec:empirical}.

\glsreset{AA}

\parhead{\gls{ML} decisions that exercise affirmative action.}  How
can we construct an algorithmic decision that exercises \gls{AA}?  We
now show how to correct the \gls{EO}-fair predictor to account for
historical disadvantages that stem from the sensitive attribute $S$.

Consider the \gls{EO}-fair predictor $f_{\eo}(a)$ and an applicant
with attributes $\{s_{\new}, a_\new\}$.  Their \gls{AA}-fair decision
is~$\hat{Y}^{\rmaa}~\sim~f_{\rmaa}(s_{\new},~a_{\new})$, where
\begin{align}
  \label{eq:counterfactualaa}
  f_{\rmaa}(s_{\new}, a_{\new}))
  &=
    \int \int
    f_{\eo}(a(s))\,\, p(a(s) \g s_{\new}, a_{\new}) \,\, p(s)
    \dif a(s) \dif s.
\end{align}
The distribution $p(s)$ is the population distribution of the
sensitive attribute; the distribution $p(a(s) \g s_{\new}, a_{\new})$
is the counterfactual distribution of the non-sensitive attribute $a$,
under intervention of $s$ and conditional on the observed values
$a_{\new}$ and $s_{\new}$.

This predictor $f_{\rmaa}(s, a)$ exercises affirmative action because
it corrects the applicant's test score (and their resulting admissions
decision) to their counterfactual test scores under intervention of
the sex.  Further, each element is computable from the dataset.  (See
below.)  Note that \gls{AA}-fair decisions can be formed from any
\gls{EO}-fair predictor.

One way to form an \gls{AA}-fair decision is to draw from the mixture
defined in \Cref{eq:counterfactualaa}.  Consider the admissions
example and an applicant $\{s_{\new}, a_{\new}\}$: (a) sample a sex
from the population distribution $s \sim p(s)$; (b) sample a test
score from its counterfactual distribution $a(s)\sim p(a(s) \g
s_{\new}, a_{\new})$; (c) sample the \gls{EO}-fair decision for that
counterfactual test score $\hat{y}^{\rmaa} \sim f_{\eo}(a(s))$.

One subtlety of $\hat{Y}^{\rmaa}$ is in step (b): it draws the
counterfactual test score $a$ under an intervened sex $s$, but
conditional on the observed sex and test score.  This is an
\textit{abduction}---if a female applicant $s_{\new}$ has a high test
score $a_{\new}$, then her counterfactual test score
$a(\textrm{male})$ will be even higher, and the likelihood of
admission goes up beyond the \gls{EO}-fair decision
\citep{pearl2009causality}.

Put differently, the \gls{AA}-fair predictor uses the \gls{EO}-fair
predictor $f_{\eo}(a)$, which only depends on the test score (and
averages over the sex). However, the \gls{AA}-fair predictor also
corrects for the bias in the test score due to the sex of the
applicant.  It replaces the current test score $a_{\new}$ with the
adjusted score $a(s) \g \{s_{\new}, a_{\new}\}$ under $s \sim
p(s)$. With this corrected score, it produces an \gls{AA}-fair
decision.

The \gls{AA}-fair predictor depends on the sensitive attribute, and
thus the \gls{AA}-fair decision $\hat{Y}^{\rmaa}$ is not
\gls{EO}-fair: its likelihood changes based on the sex of the
applicant.  But among all \gls{AA}-fair decisions, it is closest to
the \gls{EO}-fair decision.  The following theorem establishes the
  theoretical properties~of~$\hat{Y}^{\rmaa}$.

\begin{thm}[\gls{AA}-fairness and optimality of the \gls{AA} decisions
\label{thm:aathm}]
The decision $\hat{Y}^{\rmaa}\sim f_{\mathrm{eo}}(s,a)$ is
\gls{AA}-fair. Moreover, among all \gls{AA} decisions, the
\gls{AA} predictor minimally modifies the marginal distribution of
$Y^{\eo}$,
\begin{align*}
\hat{Y}^{\rmaa} =
\argmin_{Y^{\rmaa} \in\mathcal{Y}^{\rmaa}}
\mathrm{KL}(P(\hat{Y}^{\eo}(S,A))||P(Y^{\rmaa}(S,A))),
\end{align*}
where $\mathcal{Y}^{\rmaa}$ are all \gls{AA} decisions. It also
preserves the marginal distribution of the \gls{EO} predictor,
$P(\hat{Y}^{\rmaa}) = P(\hat{Y}^{\eo})$. (The proof is in
\Cref{sec:aathmpf}.)  
\end{thm}

In \Cref{thm:aathm}, we also prove that \gls{AA}-fair decisions
preserve the marginal distribution of the \gls{EO}-fair
decisions. This property makes the \gls{AA} predictor applicable as a
decision policy. If the \gls{EO}-fair predictor admits 20\% of the
applicants, the \gls{AA}-fair predictor will also admit 20\%.  We need
not worry that we admit more applicants than we can.

To finish the running example in \Cref{tab:admissions}, we exercise
\gls{AA} and calculate $f_{\rmaa}(s, a)$ from $f_{\eo}(s, a)$.  This
is the predictor that requires abduction, calculating the test score
that the applicant would have gotten had their sex been different.
Consider applicant C.  Her \gls{EO}-fair probability of acceptance is
$69\%$, but when exercising \gls{AA} it increases to $70\%$.  This
adjustment corrects for the (simulated) systemic difficulty of females
to receive test preparation and, consequently, higher scores.


\newenvironment{rcases}
  {\left.\begin{aligned}}
  {\end{aligned}\right\rbrace}

\parhead{Calculating \gls{EO} and \gls{AA} predictors.}
\Cref{alg:eoaa} provides the algorithm for calculating \gls{EO} and
\gls{AA} predictors. (We prove its correctness in
\Cref{sec:algcorrect}.) The \gls{EO} predictor adjusts the ML
predictor $f_{\acc}$; it uses the ML predictor, but marginalizes out
the sensitive attribute. The \gls{AA} predictor adjusts the \gls{EO}
predictor $f_{\eo}$; it uses an abduction step
\citep{pearl2009causality} to impute $a(s')$ from the model $g(\cdot)$
and its residual $a-g(s)$. In admissions, for example, the abduction
step infers what a female applicant's test score would have been if
she had been male, given the score that she did achieve.

\Cref{alg:eoaa} differs from existing \gls{EO} and \gls{AA} algorithms
in that it uses all available attributes. Recent works like
\citet{kusner2017counterfactual} and \citet{kilbertus2017avoiding}
avoid using descendants of the sensitive attributes to construct
\gls{AA}-fair decisions.
FairLearning~\citep{kusner2017counterfactual} performs the same
abduction step as $f_{\rmaa}$ to compute residuals
$\varepsilon_i = a_i - \E{}{A\g S=s_i}$ but then fits a predictor
using only the residuals and non-descendants of the sensitive
attributes.  For example, in admissions, after calculating the
residual it does not include the test score in its predictor.  In
contrast, \Cref{alg:eoaa} uses all available attributes and still is
\gls{AA}-fair. The \gls{EO} predictor in \Cref{alg:eoaa} also makes
use of all attributes; it does not omit sensitive attributes as done
in \gls{FTU}~\citep{kusner2017counterfactual}. In \Cref{sec:empirical}
we demonstrate empirically that \Cref{alg:eoaa} forms more accurate
decisions than these existing \gls{EO} and \gls{AA} algorithms.

\begin{algorithm}[t]
  \DontPrintSemicolon \; \KwIn{Data $\mathcal{D}=\{(s_i, a_i,
  y_i)\}_{i=1}^{n}$, where $s_i$ is sensitive, $a_i$ is not, and $y_i$
  is the decision.}

  \BlankLine

  \KwOut{Predictors $\{f_{\acc}(s, a), f_{\eo}(a), f_{\rmaa}(s, a)\}$}

  \BlankLine

  From the data $\mathcal{D}$, fit $f_{\acc}(s, a)$, $p(s)$, and
  $g(s)=\E{}{A\g S=s}$ (e.g., with regression).

  \BlankLine

  \circled{1} \quad The \gls{ML} predictor $f_{\ML}(s,a)$ draws from
  $
    \hat{y}^{\acc} \sim f_{\acc}(s, a).
    $

  \circled{2} \quad The \gls{EO} predictor $f_{\eo}(a)$ draws from \Cref{eq:counterfactualeo},
  \begin{align*}
    s' \sim p(s); \qquad \hat{y}^{\eo} \sim f_{\acc}(s',a).
  \end{align*}

  \circled{3} \quad The \gls{AA} predictor $f_{\rmaa}(s,a)$ draws from
  \Cref{eq:counterfactualaa},
  \begin{align*}
    s' \sim p(s); \qquad a' = g(s') + (a-g(s)); \qquad
    \hat{y}^{\rmaa}&\sim f_{\eo}(s',a).
  \end{align*}
  \caption{The \gls{EO} and \gls{AA} predictors.}
  \label{alg:eoaa}
\end{algorithm}

\parhead{Multiple sensitive attributes and general causal graphs.}
The admissions example involves one sensitive attribute, one
non-sensitive attribute, and a binary decision.  However,
\gls{EO}-fair and \gls{AA}-fair decisions directly apply to cases with
multiple sensitive (and non-sensitive) attributes, non-binary
decisions, and more general causal structures.

Notably, \gls{AA}-fair decisions can correct for bias when an
individual belongs to an advantaged group in one sensitive attribute
and a disadvantaged one in another. Consider both race and gender as
sensitive attributes. If an applicant comes from an advantaged racial
group but a disadvantaged gender group, the \gls{AA}-fair decision
would consider the counterfactual test score averaging over both
racial and gender groups. It corrects for both the positive racial
bias and the negative gender bias.

The \gls{EO}-fair and \gls{AA}-fair decisions also extend to settings
where the bias in some attributes need not be corrected. For example,
in affirmative action we might correct for the effect of the
applicant's sex on their test scores, but allow applicants of
different sex to differ in their interests and other abilities. The
\gls{AA}-fair predictor can calculate counterfactual $A(s)$ on only
those attributes we want to correct.



\section{Empirical studies}
\label{sec:empirical}

\glsreset{FTU}

We study algorithmic decisions on simulated and real datasets.  We
examine the fairness/accuracy trade-off of the \gls{EO} and \gls{AA}
algorithms presented here, and compare to existing algorithms that
target the same \gls{EO} and \gls{AA} criterion.  (The supplement
provides software that reproduces these studies.)

We find the following: 1) the \gls{ML} decision $\hat{Y}^{\acc}$ is
accurate but unfair; 2) The \gls{EO} decision $\hat{Y}^{\eo}$ is less
accurate than $\hat{Y}^{\acc}$, but is \gls{EO}-fair; 3) The \gls{AA}
decision $\hat{Y}^{\rmaa}$ is less accurate than the \gls{EO}-fair
decision, but is \gls{AA}-fair and achieves demographic parity; 4) The
\gls{FTU} decision $\hat{Y}^{\textrm{FTU}}$ is \gls{EO}-fair but less
accurate than $\hat{Y}^{\eo}$; 5) The FairLearning decision
$\hat{Y}^{\textrm{FL}}$ is \gls{AA}-fair and achieves demographic
parity, but less accurate than $\hat{Y}^{\rmaa}$.

\parhead{Methods and metrics.}  Each problem consists of training and
test sets, potentially with multiple sensitive and non-sensitive
attributes.  We use \Cref{alg:eoaa} to fit \gls{ML}, \gls{EO} and
\gls{AA} predictors with training data. We also use \gls{FTU} and
FairLearning, both from~\citet{kusner2017counterfactual}.  The
\gls{FTU} predictor follows classical \gls{ML} but omits the sensitive
attribute; FairLearning is described above.

With these predictors, we make algorithmic decisions $\hat{y}$ about
the $n$ test set units.  We evaluate the accuracy of these decisions relative to
the ground truth and we assess their fairness.  For a sensitive
attribute with values ``adv'' (for the advantaged group) and ``dis''
(for the disadvantaged group), we assess fairness with the following
metrics:
\begin{eqnarray*}
\textrm{\gls{EO} metric} &=& \frac{1}{n} \sum_{i=1}^{n} \left[\E{}{\hat{y}(\textrm{adv}, a_i) \g a_i, s_i} - \E{}{\hat{y}(\textrm{dis}, a_i) \g a_i, s_i}\right], \\
\textrm{\gls{AA} metric} &=& \frac{1}{n} \sum_{i=1}^{n} \left[\E{}{\hat{y}(\textrm{adv}, a(\textrm{adv})) \g a_i, s_i} - \E{}{\hat{y}(\textrm{dis}, a(\textrm{dis})) \g a_i, s_i}\right].
\end{eqnarray*}
When the \gls{EO} metric is close to 0.0, we achieve
\gls{EO}-fairness. Less than 0.0 indicates a bias towards the
disadvantaged group. More than 0.0 indicates a bias towards the
advantaged group. The same interpretation applies to the \gls{AA}
metric.

\parhead{Simulation studies.} We first study datasets that simulate
the unfair admissions committee from \Cref{eq:admissions-model}.  We
fix the effect of test score on admissions to $\beta_a = 2.0$.  We
generate multiple datasets by varying the committee's bias due to sex
$\beta_s$ and the historical disadvantage on test score $\lambda$.

\Cref{fig:varybetaS} show how
\gls{EO}-fairness and predictive accuracy trade off as the bias
$\beta_s$ increases.  Only the \gls{EO} predictor and \gls{FTU}
achieve \gls{EO} fairness. Although both are less accurate than
classical \gls{ML}, the \gls{EO} predictor achieves greater accuracy than
\gls{FTU}.  \Cref{fig:varybetaS} show
how \gls{AA}-fairness and predictive accuracy trade off as the
historical disadvantage $\lambda$ increases.  Only \gls{AA} and
FairLearning achieve \gls{AA}-fairness.  Among these \gls{AA}-fair
predictors, the \gls{AA} predictor is more accurate.

\parhead{Case studies.} We study the fair algorithms on three real
datasets: 1) Adult income for predicting whether individuals' income
is higher than \$50K, 2) ProPublica's COMPAS for predicting recidivism
scores 3) German credit for predicting whether individuals have good
or bad credit. For Adult and COMPAS, both gender and race are
sensitive attributes. For German credit, gender and marital status are
sensitive attributes.
Decisions are binary except for real-valued COMPAS recidivism scores.

\begin{figure}[t]
		\begin{minipage}{.5\textwidth}
		\centering
		\captionsetup{width=.9\linewidth}
		\includegraphics[scale=0.6]{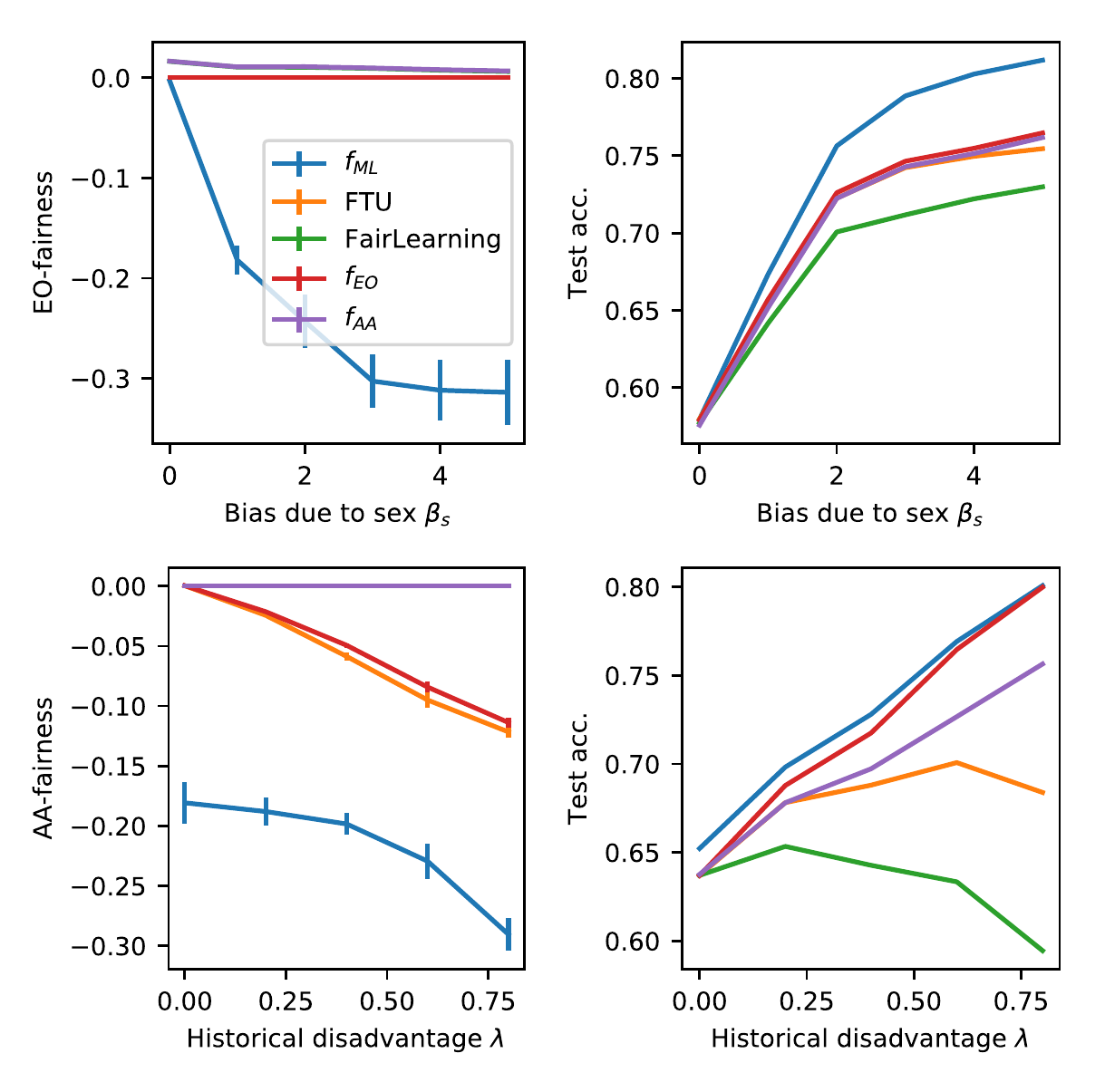}
	\end{minipage}
	\begin{minipage}{.5\textwidth}
		\centering
		\captionsetup{width=.9\linewidth}

  \footnotesize
  \begin{center}
    \begin{tabular}{lcccc} 
      \toprule
      &\multicolumn{4}{c}{\textbf{Metrics ($\times 10^{2}$) on Adult}} \\
      &\textbf{\gls{EO}} & \textbf{\gls{AA}} & \textbf{KL} & \textbf{Prediction} \\
      \midrule
       $f_{\acc}$ & 2.5 & 16.2 & 15.0 & 78.6 \\
      \midrule
      \gls{FTU} & \textbf{0} & 14.8 & 12.2 & 77.3\\
     $f_\eo$ & \textbf{0} & 14.1 & 12.6 & \textbf{77.4}\\
      \midrule
      FL  & -14.8 & \textbf{0}& 5.3 & 75.1\\
      $f_\rmaa$ & -9.1 & \textbf{0} & \textbf{1.5} & \textbf{77.1} \\
      \bottomrule
    \end{tabular}
  \end{center}



		  \caption{Both the \gls{EO} predictor $f_\eo$ and \gls{FTU} are \gls{EO}-fair.
			The \gls{AA} predictor $f_\rmaa$ and FairLearning (FL)
			\citep{kusner2017counterfactual} are \gls{AA}-fair and achieve demographic parity (close-to-zero KL).
			The \gls{ML} predictor $f_{\acc}$ predicts best overall but the
			\gls{EO} predictor predicts best among the \gls{EO}-fair
			predictors. 
			We report mean values across individuals. \gls{EO} and
			\gls{AA} metric standard deviations are $\leq 0.1$ and
			$\leq 0.11$, respectively.
			\label{tab:adultmetrics}}
	\end{minipage}
        \caption{Varying $\beta_{s}$ from $0.0$ to $+5.0$; varying
          $\lambda$ from $0.0$ to $+0.8$; error bars indicate +/- 1
          sd. $f_{\acc}$ is not fair but predicts best; $f_\eo$ is the
          best \gls{EO}-fair predictor; $f_\rmaa$ is the best
          \gls{AA}-fair predictor.) \label{fig:varybetaS}}
      \end{figure}

\Cref{tab:adultmetrics} reports results on the Adult dataset.  (In
\Cref{sec:empiricalsupp}, we report results on the COMPAS and German
credit datasets, and provide all the details needed to reproduce our
studies.) The table reports accuracy, \gls{EO}, and \gls{AA}. Further,
to measure demographic parity, it reports the KL divergence between
decisions $p(\hat{y}(s, a))$ and $p(\hat{y}(s', a))$; when there is
demographic parity, the KL is close~to~zero.

The findings are consistent with the simulation studies. The \gls{EO}
and \gls{FTU} predictors achieve \gls{EO}-fairness and, among these,
the \gls{EO} predictor is more accurate. The \gls{AA} predictor and
FairLearning achieve \gls{AA}-fairness and demographic parity; but the
\gls{AA} predictor is more accurate because it uses all of the
attributes.  While the \gls{ML} predictor has the best accuracy, it is less
\gls{EO}-fair and \gls{AA}-fair.


\section{Discussion}
\label{sec:discussion}

We develop fair ML algorithms that modify fitted ML predictors to make
them fair.  We prove that the resulting predictors are \gls{EO}-fair
or \gls{AA}-fair, and that they otherwise maximally recover the fitted
ML predictor.  We show empirically that the \gls{EO} and \gls{AA}
predictors produce better predictions than existing algorithms for
satisfying the same fairness criteria.

The \gls{EO} and \gls{AA} predictors both aim to recover the ML
outcome as much as possible. However, they can also readily extend to
recover the ground truth outcome subject to availability of data. For
example, instead of predicting the admissions decision accurately, we
can aim the \gls{EO} and \gls{AA} predictors for the ground truth
outcome, i.e. first year GPA. They will lead to fair and accurate
predictions of whether the applicants' succeed in college.

\clearpage
\putbib[BIB1]
\end{bibunit}

\clearpage
\begin{bibunit}[alp]
{\onecolumn

\appendix
\onecolumn
{\Large\textbf{Appendix}}

\section{Proof of \Cref{thm:eothm}}
\label{sec:eothmpf}

Before proving the \gls{EO}-fairness and optimality of the
\gls{EO} predictor, we first establish a lemma about
\gls{EO}-fairness. It says the \gls{EO}-fairness of a decision
$\hat{Y}$ is equivalent to no causal arrow from $S_i$ to $\hat{Y}$.

\begin{lemma} (\gls{EO} $\Leftrightarrow$ No $(S\rightarrow \hat{Y})$)
\label{lemma:EOequiv} Assume the causal graph in \Cref{fig:eoaagraph}. A
decision $\hat{Y}$ satisfies equal opportunities over
$S$ if and only if there is no causal arrow between $S$ and $\hat{Y}$.
\end{lemma}

Consider a decision $\hat{Y}(s, a)$ in the causal
graphical model \Cref{fig:yhat}. The reason behind
\Cref{lemma:EOequiv} is that $\hat{Y}(s,a)\perp S, A$ in
\Cref{fig:yhat} \citep{pearl2009causality}. So \gls{EO} reduces to
$\hat{Y}(s,a)$ being constant in $s$.

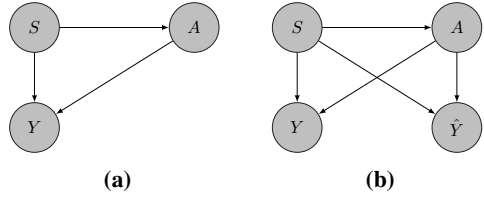
\begin{figure}[ht]
\centering
\begin{subfigure}[b]{0.25\textwidth}
\centering
\begin{adjustbox}{height=2cm}
\begin{tikzpicture}
[obs/.style={draw,circle,filled,minimum size=1cm}, latent/.style={draw,circle,minimum size=1cm}]
  \node[obs]                               (t1) {$A$};
  \node[obs, left=of t1, xshift=-1.2cm]                               (t2) {$S$};
  \node[obs, below=of t2]            (y) {$Y$};  
  \edge {t2} {y, t1} ; 
  \edge {t1} {y} ;

\end{tikzpicture}
\end{adjustbox}
\caption{\label{fig:setupmain}}
\end{subfigure}%
\begin{subfigure}[b]{0.25\textwidth}
\centering
\begin{adjustbox}{height=2cm}
\begin{tikzpicture}
[obs/.style={draw,circle,filled,minimum size=1cm}, latent/.style={draw,circle,minimum size=1cm}]
  \node[obs]                               (t1) {$A$};
  \node[obs, left=of t1, xshift=-1.2cm]                               (t2) {$S$};
  \node[obs, below=of t2]            (y) {$Y$};  
  \node[obs, below=of t1]            (yeo) {$\hat{Y}$};
  \edge {t2} {y, t1, yeo} ; 
  \edge {t1} {y, yeo} ;

\end{tikzpicture}
\end{adjustbox}
\caption{\label{fig:yhat}}
\end{subfigure}
\caption{The causal graph for the sensitive attribute $S$, the
	 attribute $A$, the past decisions $Y$, and the
predicted decisions $\hat{Y}$.}
\end{figure}

\Cref{lemma:EOequiv} shows that the decision $\hat{Y}$ is
\gls{EO}-fair as long as $\hat{Y}$ does not predict using the
sensitive attribute.

\begin{proof}
The goal is to show that counterfactual fairness is equivalent to
\begin{align}
\hat{Y}^{\mathrm{eo}}(s,a) = \hat{Y}^{\mathrm{eo}}(a)  
\end{align} 
for any $a\in \mathcal{A}, s, s'\in \mathcal{S}$. This equation is
equivalent to no causal arrow between $S$ and $\hat{Y}_{\mathrm{eo}}$. 

We start with the definition of \gls{EO}-fairness.
\begin{align}
P(\hat{Y}(s', a)\g S=s, A=a,) &= P(\hat{Y}(s, a)\g S=s, A=a)\\
\Leftrightarrow P(\hat{Y}(s, a')) &= P(\hat{Y}(s, a)) \qquad \forall a, s, s'\label{eq:keyeo}\\
\Leftrightarrow P(\hat{Y}(s', a)) &= P(\hat{Y}(a))\label{eq:laststep}
\end{align}

\Cref{eq:keyeo} is due to the observation that $\hat{Y}(s,a)\perp
A,S$. \Cref{eq:laststep} is true because 
\begin{align}
\hat{Y}(a) = \int \hat{Y}(s, a) p_S(s)\dif s = \hat{Y}(s', a)
\end{align}
for any $s'$. The last equation is equivalent to no causal arrow from
$S$ to $\hat{Y}$.

\end{proof}

Now we are ready to prove \Cref{thm:eothm}.
\begin{proof}
As a direct application of \Cref{lemma:EOequiv}, we prove the first
part of \Cref{thm:eothm}: $\hat{Y}^{\mathrm{\mathrm{eo}}}$ is
\gls{EO}-fair. The reason is that
$\hat{Y}^{\mathrm{\mathrm{eo}}}(s_{\mathrm{new}}, a_{\mathrm{new}}) =
\hat{Y}^{\mathrm{\mathrm{eo}}}(a_{\mathrm{new}})$ as is defined in
\Cref{eq:counterfactualeo}.

We then prove the second part of \Cref{thm:eothm}. It establishes the
optimality of $\hat{Y}^{\mathrm{\mathrm{eo}}}$.

We start with rewriting the goal of the proof. We will show this goal is equivalent to the definition of the \gls{EO} predictor.

\begin{align}
\hat{Y}^{\mathrm{\mathrm{eo}}}=&\argmin_{Y^{\rm{\mathrm{eo}}}\in\mathcal{Y}^{\rm{\mathrm{eo}}}} \E{S,A}{\mathrm{KL}(P(\hat{Y}^{\mathrm{ml}}(S,A))||P(Y^{\rm{\mathrm{eo}}}(S,A))}\label{eq: step1}\\
=& \argmin_{Y^{\rm{\mathrm{eo}}}\in\mathcal{Y}^{\rm{\mathrm{eo}}}} \E{A}{\int P(S)\int P(\hat{Y}^{\mathrm{ml}}(S,A))\left[\log P(\hat{Y}^{\mathrm{ml}}(S,A)) - \log P(Y^{\rm{\mathrm{eo}}}(S,A))\right]}\dif y\dif s \label{eq: step2}\\
= & \argmin_{Y^{\rm{\mathrm{eo}}}\in\mathcal{Y}^{\rm{\mathrm{eo}}}} - \E{A}{\int P(S)\int P(\hat{Y}^{\mathrm{ml}}(S,A)) \log P(Y^{\rm{\mathrm{eo}}}(S,A))\dif y\dif s}\label{eq: step3}\\
= & \argmin_{Y^{\rm{\mathrm{eo}}}\in\mathcal{Y}^{\rm{\mathrm{eo}}}} - \E{A}{\int P(S)\int P(\hat{Y}^{\mathrm{ml}}(S,A)) \log P(Y^{\rm{\mathrm{eo}}}(A))\dif y\dif s}\label{eq: step4}\\
= & \argmin_{Y^{\rm{\mathrm{eo}}}\in\mathcal{Y}^{\rm{\mathrm{eo}}}} - \E{A}{\int \left[\int P(S)P(\hat{Y}^{\mathrm{ml}}(S,A)) \dif s\right]\log P(Y^{\rm{\mathrm{eo}}}(A))\dif y}\label{eq: step5}\\
= & \argmin_{Y^{\rm{\mathrm{eo}}}\in\mathcal{Y}^{\rm{\mathrm{eo}}}} - \E{A}{\int P(\hat{Y}^{\mathrm{ml}}\s \rm{do}(A=A))\log P(Y^{\rm{\mathrm{eo}}}(A))\dif y}\label{eq: step6}\\
=& \argmin_{Y^{\rm{\mathrm{eo}}}\in\mathcal{Y}^{\rm{\mathrm{eo}}}} \E{A}{\mathrm{KL}(P(\hat{Y}^{\mathrm{ml}}\s \rm{do}(A=A))|| P(Y^{\rm{\mathrm{eo}}}(A))}\label{eq: step7}\\
= & P(\hat{Y}^{\mathrm{ml}}\s \rm{do}(a))\label{eq: step8}\\
=& \int p(\hat{y}^{\mathrm{ml}}(s,a)) p(s)\dif s \label{eq: step9}
\end{align}

\Cref{eq: step1} is the goal of the proof. \Cref{eq: step2} is due to
the definition of the \gls{KL} divergence. \Cref{eq: step3} is due to
$\hat{Y}^{\mathrm{acc}}$ being a given random variable. \Cref{eq:
step4} is due to \Cref{lemma:EOequiv}. \Cref{eq: step5} switches the
integral subject to conditions of the dominated convergence theorem.
\Cref{eq: step6} is due to \Cref{fig:yhat} and the backdoor adjustment
formula of \citet{pearl2009causality}. \Cref{eq: step7} is due to the
definition of the \gls{KL} divergence and $\hat{Y}^{\mathrm{ml}}$
being given. \Cref{eq: step8} is because setting $P(Y^{\mathrm{eo}}(a)) =
P(\hat{Y}^{\mathrm{ml}}\s \rm{do}(a))$ has
$\mathrm{KL}(P(\hat{Y}^{\mathrm{ml}}\s \rm{do}(A=a))||
P(Y^{\rm{\mathrm{eo}}}(a)) = 0$ for all $a$. The expecation is hence also zero:
$\E{A}{\mathrm{KL}(P(\hat{Y}^{\mathrm{ml}}\s \rm{do}(A=A))||
P(Y^{\rm{\mathrm{eo}}}(A))} = 0$.  \Cref{eq: step9} is due to the definition of
the intervention distribution \citep{pearl2009causality}.

This calculation implies $\hat{Y}^{\rm{\mathrm{eo}}} =  \int
p(\hat{y}^{\mathrm{ml}}(s, a)) p(s)\dif s$ minimizes the average
\gls{KL} between the ML decision and the \gls{EO} decision $\E{A,
S}{\mathrm{KL}(P(\hat{Y}^{\mathrm{ml}}(S,A))||P(\hat{Y}^{\rm{\mathrm{eo}}}(S,A))}$.
the \gls{EO} predictor maximally recovers the ML decision. Put
differently, the \gls{EO} predictor minimally modifies the ML decision
to achieve \gls{EO}-fairness.

\end{proof}

\section{Proof of \Cref{thm:aathm}}

\label{sec:aathmpf}

\begin{proof}

We first prove that the \gls{AA} predictor is \gls{AA}-fair.

Recall the definition of the \gls{AA} predictor: 
\begin{align}
p(\hat{y}^{\mathrm{aa}}(s_{\mathrm{new}}, a_{\mathrm{new}})) = 
      \int \int
      p(\hat{y}^{\mathrm{\mathrm{eo}}}(a(s^0)))\,\, p(a(s^0) \g s_{\mathrm{new}}, a_{\mathrm{new}}) \,\, p(s^0) \dif a(s^0)
      \dif s^0.
\end{align}

Note the above equation is the same definition as in
\Cref{eq:counterfactualaa}. We simply change $s$ to $s^0$ for
downstream notation convenience.

This implies
\begin{align}
p(\hat{y}^{\mathrm{aa}}(s', a(s'))) = \int \int
      p(\hat{y}^{\mathrm{\mathrm{eo}}}(a(s^0)))\,\, p(a(s^0) \g s', a(s')) \,\, p(s^0) \dif a(s^0)
      \dif s^0.
\end{align}

Hence we have
\begin{align}
&p(\hat{y}^{\mathrm{aa}}(s', a(s'))\g S=s, A=a) \\
=& \int \int
      p(\hat{y}^{\mathrm{\mathrm{eo}}}(a(s^0)))\,\, p(a(s^0) \g s, a(s')') \,\, p(s^0) p(a(s')\g s, a)
      \dif a(s')
      \dif a(s^0)
      \dif s^0\\
=&\int \int
      p(\hat{y}^{\mathrm{\mathrm{eo}}}(a(s^0)))\,\, p(a(s^0) \g s, a) \,\, p(s^0) 
      \dif a(s^0)
      \dif s^0.\label{eq:aastep}
\end{align}
\Cref{eq:aastep} is due to $\int p(a(s^0) \g s', a(s'))p(a(s')\g s, a)
\dif a(s') = p(a(s^0) \g s,a) $ by the structural equation implied by the causal graph:
\begin{align}
A \stackrel{a.s.}{=} f(S, \epsilon).
\end{align} 

The intuition here is the observed $s,a$ will provide the same
information as $s', a(s'), $ for the same person. They all contain the
information about the background variable that affects $A$. We hence
have $p(a(s^0) \g s', a(s')) = p(a(s^0) \g a(s'), s', a, s) = p(a(s^0)
\g a(s'), a, s) $.

Notice that the right hand side of \Cref{eq:aastep} does not depend on
$s'$. This implies $p(\hat{y}^{\mathrm{aa}}(s', a(s'))\g S=s, A=a) =
p(\hat{y}^{\mathrm{aa}}(a(s), s)\g A=a, S=s)$. We simply repeat the
same calculation for $p(\hat{y}^{\mathrm{aa}}(s, a(s))\g S=s, A=a)$.

We hence have proved that the \gls{AA} predictor is \gls{AA}-fair:
\begin{align} P(\hat{Y}^{\mathrm{aa}}(s', A(s'))\g S= s, A = a) =
P(\hat{Y}^{\mathrm{aa}}(s, A(s))\g S= s, A = a).
\end{align}
It establishes the first part of \Cref{thm:aathm}.

We can also prove demographic parity for $\hat{Y}^{\mathrm{aa}}$.

\begin{align}
&P(\hat{Y}^{aa}(S, A)\g S) \\
= &\int P(Y^{\mathrm{eo}}(a')) P(A(s')=a'\g S,
A)P(s')\dif a'\dif s' P(A\g S) \dif A\\
= &\int P(Y^{\mathrm{eo}}(a')) P(A(s')=a'\g S)P(s')\dif a'\dif s'  \\
= &\int P(Y^{\mathrm{eo}}(a')) P(A(s')=a')P(s')\dif a'\dif s'  \\
\end{align}

The third equality is due to $A(s')\perp S$ by Pearl's twin network \citep{pearl2009causality}.

We next compute the
marginal distribution of $\hat{Y}^{aa}$.
\begin{align}
&P(\hat{Y}^{aa}(S, A)) \\
= &\int \int P(Y^{\mathrm{eo}}(a')) P(A(s')=a'\g S,
A)P(s')\dif a'\dif s' P(S, A) \dif A\dif S\\
= &\int P(Y^{\mathrm{eo}}(a')) P(A(s')=a')P(s')\dif a'\dif s'  \\
\end{align}

Therefore, we have
\begin{align}
P(\hat{Y}^{aa}(S, A)\g S) = P(\hat{Y}^{aa}(S, A)).
\end{align}
Hence,
\begin{align}
\hat{Y}^{aa}(S, A)\perp S.
\end{align}

This establishes demographic parity.

We next prove the second part of \Cref{thm:aathm}. We show that
$Y^{aa}$ minimizes $\mathrm{KL}(P(Y^{\mathrm{\mathrm{eo}}})||P(\hat{Y}^{\mathrm{aa}}))$. In
fact, $Y^{aa}$ recovers the marginal distribution of $\hat{Y}^{\mathrm{eo}}$.

\begin{align}
P(\hat{Y}^{aa}) = &P(\hat{Y}^{aa}(S, A)) \\
= &\int P(Y^{\mathrm{eo}}(a')) P(A(s')=a')P(s')\dif a'\dif s'  \\
= &\int P(Y^{\mathrm{eo}}(a')) P(A(s')=a')P(s')\dif a'\dif s'  \\
= &\int P(Y^{\mathrm{eo}}(a')) P(A=a')\dif a'  \\
= & P(\hat{Y}^{\mathrm{eo}}(A))\\
= & P(\hat{Y}^{\mathrm{eo}})
\end{align}

This gives $\mathrm{KL}(P(Y^{\mathrm{\mathrm{eo}}})||P(\hat{Y}^{\mathrm{aa}}))
= 0$. Hence, $\hat{Y}^{aa}$ minimizes this \gls{KL} and preserves the marginal distribution of $\hat{Y}^{\mathrm{eo}}$.

\end{proof}

\section{The correctness of \Cref{alg:eoaa}}

\label{sec:algcorrect}

We leverage the backdoor adjustment formula in
\citet{pearl2009causality} to compute $\hat{Y}^{\rm{eo}}$ from the
past admissions records $\{(S_i, A_i, Y_i)\}_{i=1}^n$.
\begin{prop}
\label{prop:compute-yeo}
Assume the causal graph \Cref{fig:setupmain}. The \gls{EO} predictor
can be computed using the observed data $(S_i, A_i, Y_i)$:
\begin{align}
\label{eq:compute-yeo}
P(\hat{Y}^{\rm{eo}}(s_{\rm{new}}, a_{\rm{new}})) = 
\int P(Y_i\g S_i=s', A_i = a_{_{\rm{new}}}) P(S_i=s')\dif s',
\end{align}
if 
\begin{align}
\label{eq:overlap}
P(A_i \in \mathbf{A} \g S_i=s) > 0
\end{align} for all $P(\mathbf{A}) > 0$,
$s\in \mathcal{S}$ and $\mathbf{A}\subset \mathcal{A}.$
\end{prop}
\begin{proof}
\Cref{prop:compute-yeo} is a direct consequence of Theorem 3.3.2 of
\citet{pearl2009causality}. 
\end{proof}
\Cref{eq:compute-yeo} reiterates the fact that $\hat{Y}^{\rm{eo}}$
does not simply ignore the sensitive attribute $S$; estimating
$P(Y_i\g S_i, A_i)$ relies on $S$ in the training data. However,
$\hat{Y}^{\rm{eo}}(s_{\rm{new}}, a_{\rm{new}})$ does not rely on
$s_{\rm{new}}$ in the test data.

We leverage the abduction-action-prediction approach
\citep{pearl2009causality} to compute $\hat{Y}^{\rm{aa}}$.
\begin{prop}
\label{prop:compute-yaa} 
Assume the causal graph \Cref{fig:setupmain}. Write
\begin{align}
A \stackrel{a.s}{=} f_A(S, \epsilon_A)
\end{align}
for some function $f_A$ and zero mean random variable $\epsilon_A$
satisfying $S\perp \epsilon_A$. The \gls{AA} predictor can be
computed using the observed data $(S_i, A_i, Y_i)$:
\begin{align}
&P(\hat{Y}^{\rm{aa}}(s, a) ) \nonumber\\ 
= & \int P(Y^{\rm{eo}}(s, a')) \cdot P(A=a'\g S=s',
\epsilon_A)\cdot P(\epsilon_A\g S=s, A=a)\cdot P(s')\dif a'\dif s'\label{eq:abduction}\\
\approx & \int P\left(Y^{\rm{eo}}\left(\E{}{A_i\g S_i=s'} + \hat{\epsilon}_A\right), s\right) \cdot P(s')\dif s'\label{eq:linapproxabduction}
\end{align}
where $\hat{\epsilon}_A = a- \E{}{A_i\g S_i=s}$.
\end{prop}

\begin{proof}
\Cref{eq:abduction} is a direct consequence of Theorem 7.1.7 of
\citet{pearl2009causality}. \Cref{eq:linapproxabduction} is due to a
linear approximation of $f_A$. 
\end{proof}

\Cref{alg:eoaa} can easily generalize to cases where sensitive
attributes are unobserved. In these cases, we can leverage recent
techniques in multiple causal inference like the approach from
\citet{wang2018blessings}.  These approaches construct substitutes for
the unobserved sensitive attributes from the descendants of the
sensitive attributes.

\section{Detailed results of the empirical studies}
\label{sec:empiricalsupp}

We study the fair algorithms on three real datasets:
\begin{itemize}[leftmargin=*]
\item \parhead{Adult income dataset.\footnote{https://www.kaggle.com/wenruliu/adult-income-dataset}} The task is to
  predict whether someone has a yearly income higher than \$50K; the
  decision is binary. The sensitive attributes are gender and race;
  other attributes include education and marital status. The dataset
  has 32,561 training samples and 16,282 test samples.

\item \parhead{ProPublica's COMPAS recidivism data.} The
  task is to predict an individual's recidivism score; the decision is
  real-valued.  The sensitive attributes are gender and race; other
  attributes include the priors count, juvenile felonies count, and
  juvenile misdemeanor count. The dataset has 6,907 complete
  samples. We split them into 75\% training and 25\% testing.

\item \parhead{German credit data.\footnote{https://www.kaggle.com/uciml/german-credit}} The task is to
  predict whether an individual has good or bad credit; the decision is
  binary. The sensitive attributes are gender and marital status;
  other attributes include credit history, savings, and employment
  history. The dataset has 1,000 samples. We split them into 75\%
  training and 25\% testing.
\end{itemize}

We report the detailed results of the COMPAS dataset and the German
credit dataset in \Cref{tab:compasmetrics} and
\Cref{tab:creditmetrics}.


\begin{table}[t]
	\footnotesize		
	\centering
		\begin{tabular}{lcccc} 
			\toprule
			&\multicolumn{4}{c}{\textbf{Metrics ($\times 10^{-2}$) on COMPAS}} \\
			\midrule
			&\textbf{\gls{EO}} & \textbf{\gls{AA}} & \textbf{KL} & \textbf{Prediction} \\
			\midrule
			ML predictor $f_\ML$ & -68.1& -104.9 & 17.1 & 28.0 \\
			\cdashlinelr{1-5}
			\gls{FTU} & \textbf{0} & -52.9 & 0.5 & \textbf{25.6} \\
			\gls{EO} predictor $f_\eo$ & \textbf{0} & -36.8 & 0.5 & \textbf{25.6}\\
			\cdashlinelr{1-5}
			FairLearning  & -41.2 &  \textbf{0} & \textbf{0.2} & \textbf{22.7}\\
			\gls{AA} predictor $f_\rmaa$ & -41.2 & \textbf{0} & \textbf{0.2} & \textbf{22.7} \\
			\bottomrule
		\end{tabular}
		\vspace{5pt}
		\caption{Both the \gls{EO} predictor $f_\eo$ and \gls{FTU} are \gls{EO}-fair; they achieve zero in the \gls{EO} metric (lower is better).
			\gls{AA} predictor $f_\rmaa$ and FairLearning \cite{kusner2017counterfactual} are \gls{AA}-fair; they achieve zero in the \gls{AA} metric (lower is better).
			\gls{AA} predictor $f_\rmaa$ achieves demographic parity; it has close-to-zero \gls{KL} divergence (Lower is better.)
			ML predictor $f_\ML$ predicts best; \gls{EO} predicts best among the
			fair predictors (higher prediction scores are better.) We report mean values across individuals. \gls{EO} and \gls{AA} metric
			standard deviations are $\leq 0.42$ and $\leq 0.6$, respectively.
			\label{tab:compasmetrics}}
\end{table}

\begin{table}[t]
	\footnotesize	
		\centering	
	\begin{tabular}{lcccc} 
		\toprule
		&\multicolumn{4}{c}{\textbf{Metrics ($\times 10^{-2}$) on German Credit}} \\
		\midrule
		&\textbf{\gls{EO}} & \textbf{\gls{AA}} & \textbf{\gls{KL}} & \textbf{Prediction} \\
		\midrule
		ML predictor $f_\ML$ & -5.4 & -3.9 & 18.6 & 64.7 \\
		\cdashlinelr{1-5}
		\gls{FTU} & \textbf{0} & -2.0 & 15.6 & 63.4 \\
		\gls{EO} predictor $f_\eo$ & \textbf{0} & 1.5 & 13.2 & \textbf{64.5} \\
		\cdashlinelr{1-5}
		FairLearning  & -1.4 & \textbf{0} & 8.3 & 63.5\\
		\gls{AA} predictor $f_\rmaa$ & -1.3 & \textbf{0} &  \textbf{7.8} & \textbf{64.3} \\
		\bottomrule
	\end{tabular}
		\vspace{5pt}
	\caption{Both \gls{EO} predictor $f_\eo$ and \gls{FTU} are \gls{EO}-fair; they achieve zero in the \gls{EO} metric (lower is better).
		\gls{AA} predictor $f_\rmaa$ and FairLearning \cite{kusner2017counterfactual} are \gls{AA}-fair; they achieve zero in the \gls{AA} metric (lower is better).
		\gls{AA} predictor $f_\rmaa$ achieves demographic parity; it has close-to-zero \gls{KL} divergence (Lower is better.)
		ML predictor $f_\ML$ predicts best; \gls{EO} predicts best among the
		fair predictors (higher prediction scores are better.) We report mean values across individuals. \gls{EO} and \gls{AA} metric
		standard deviations are $\leq 0.02$ and $\leq 0.02$, respectively.
		\label{tab:creditmetrics}}
\end{table}

\begin{figure}[t]
\centering
\begin{subfigure}[b]{0.33\textwidth}
\centering
\includegraphics[width=\textwidth]{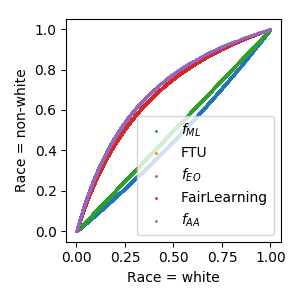}
\caption{\gls{EO} metric \label{fig:cfdecisionfixA}}
\end{subfigure}
\begin{subfigure}[b]{0.33\textwidth}
\centering
\includegraphics[width=\textwidth]{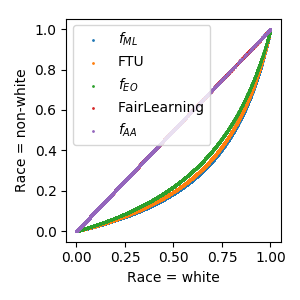}
\caption{\gls{AA} metric\label{fig:cfaadecisions}}
\end{subfigure}%
\caption{The \gls{EO} predictor satisfies the \gls{EO} criterion;
  the \gls{AA} predictor satisfies the \gls{AA} criterion. We compare
  counterfactual predictions for all individuals had they been white
  or non-white. (a) relates to the \gls{EO} metric: it holds the
  attribute values fixed. (b) allows attribute values to change after
  intervening on race.}
\end{figure}

We further discuss \gls{EO}-fairness and \gls{AA}-fairness. Consider
the \gls{EO}-fairness against race. For the Adult dataset,
\Cref{fig:cfdecisionfixA} plots $\mathbb{E}[\hat{Y}(s, a)\g S=s, A=a]$
against $\mathbb{E}[\hat{Y}(s', a)\g S=s, A=a]$, where $s$ and $s'$
only differ by the individual's race being white or non-white. A
method is \gls{EO}-fair if the predictions align with the diagonal.
Both the \gls{EO} predictor and \gls{FTU} align with the diagonal;
they are \gls{EO}-fair. None of classical ML, FairLearning, or the
\gls{AA} predictor are \gls{EO}-fair, and counterfactual
\gls{AA} and FairLearning are less \gls{EO}-fair than classical ML.

Now consider the \gls{AA}-fairness against race.  For the Adult
dataset, \Cref{fig:cfaadecisions} plots $\mathbb{E}[\hat{Y}(s, A(s))\g
S=s, A=a]$ against $\mathbb{E}[\hat{Y}(s', A(s'))\g S=s, A=a]$ when
$s$ and $s'$ differ by whether the individual is white or non-white.
If the predictions align with the diagonal, the decision is
\gls{AA}-fair. Both the \gls{AA} predictor and FairLearning align
with the diagonal; they are counterfactually fair. Counterfactual
\gls{EO} is less unfair than
\gls{FTU}. The classical ML predictor is most unfair. (These results
generalize across datasets.)

\begin{figure}
\centering
\begin{subfigure}[b]{0.33\textwidth}
\centering
\includegraphics[scale=0.5]{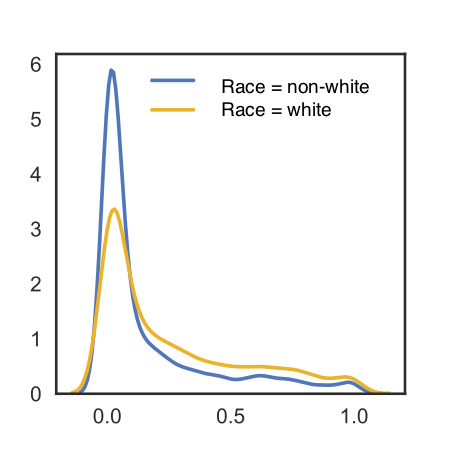}
\caption{The ML predictor\label{fig:logregdecisions}}
\end{subfigure}
\begin{subfigure}[b]{0.33\textwidth}
\centering
\includegraphics[scale=0.5]{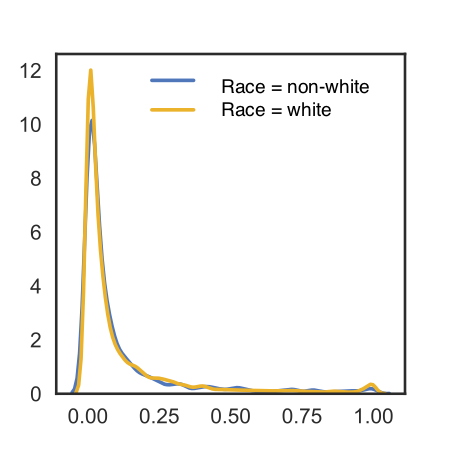}
\caption{The \gls{AA} predictor\label{fig:cfaadist}}
\end{subfigure}%
\caption{ The decision distributions of white and non-white
  individuals.  The \gls{AA} predictor produces equal distributions;
  it achieves demographic parity. \label{fig:distdiff}}
\end{figure}

Finally, we discuss demographic parity.  Demographic parity is a
group-level statistical measure that requires decisions to be
independent of sensitive attributes.  It has been used to measure
affirmative action \citep{dwork2012fairness}.  To evaluate demographic
parity, we compare the prediction distributions between the groups of
individuals with different sensitive attributes. The metric is the
symmetric \gls{KL} divergence between $P(\hat{Y} \g S=s)$ and
$P(\hat{Y} \g S=s')$. For a predictor achieving demographic parity,
the \gls{KL} is zero. (We evaluate the symmetric \gls{KL} by binning
the values of predictions.)

\Cref{tab:adultmetrics} present the symmetric \gls{KL} between the two
gender groups in the data. Both the \gls{AA} predictor and
FairLearning have close-to-zero symmetric \gls{KL}. The
\gls{AA} predictor generally has lower symmetric \gls{KL}; it is
closer to demographic parity. None of the other methods are close.

Now consider the demographic parity against race.
\Cref{fig:distdiff} demonstrates the prediction distributions of white
and non-white individuals. While classical ML has very different
prediction distributions for the two groups, the prediction
distributions of the \gls{AA} predictor is nearly identical.

}
\clearpage
\putbib[BIB1]
\end{bibunit}

\end{document}